\documentclass{article}

\usepackage{nips14submit_e,times} 

\usepackage{times}
\usepackage{relsize}
\usepackage[T1]{fontenc} 
\usepackage[latin1]{inputenc} 
\usepackage[english]{babel}

\usepackage{epsfig}
\usepackage{graphicx}
\usepackage{wrapfig}
\usepackage{subfigure}
\usepackage[belowskip=0pt,aboveskip=0pt,font=small]{caption}
\usepackage{amsmath, amsthm, amssymb}
\usepackage{textcomp}
\usepackage{stmaryrd}
\usepackage{upgreek}
\usepackage{bm}
\usepackage{cases}

\usepackage[numbers]{natbib}
\usepackage[pagebackref=true,breaklinks=true,letterpaper=true,colorlinks,bookmarks=false,citecolor=red]{hyperref}
\usepackage{bibspacing}
\usepackage{algorithm}
\usepackage{algorithmic}

\usepackage{multirow}
\usepackage{rotating}
\usepackage{booktabs}

\usepackage{enumitem}
\usepackage[olditem,oldenum]{paralist}

\usepackage{alltt}
\usepackage{listings}

\belowdisplayskip \abovedisplayskip

\newlength{\sectionReduceTop}
\newlength{\sectionReduceBot}
\newlength{\subsectionReduceTop}
\newlength{\subsectionReduceBot}
\newlength{\abstractReduceTop}
\newlength{\abstractReduceBot}
\newlength{\captionReduceTop}
\newlength{\captionReduceBot}
\newlength{\subsubsectionReduceTop}
\newlength{\subsubsectionReduceBot}

\newlength{\eqnReduceTop}
\newlength{\eqnReduceBot}

\newlength{\horSkip}
\newlength{\verSkip}

\newlength{\figureHeight}
\usepackage{mysymbols}

\usepackage{url}
\usepackage{xspace}
\usepackage{comment}
\usepackage{color}
\usepackage{afterpage}

\definecolor{green}{rgb}{0,0.5,0}

\definecolor{blue}{rgb}{0,0,0.8}

\definecolor{cyan}{cmyk}{1,0,0,0}

\definecolor{magenta}{cmyk}{0,1,0,0}

\def\vec#1{\mathchoice%
	{\mbox{\boldmath $\displaystyle\bf#1$}}
	{\mbox{\boldmath $\textstyle\bf#1$}}
	{\mbox{\boldmath $\scriptstyle\bf#1$}}
	{\mbox{\boldmath $\scriptscriptstyle\bf#1$}}}
\def\v#1{\protect\vec #1}

\newcommand{\inter}{\cap}
\newcommand{\union}{\cup}

\newcommand{\Cut}{\mathrm{Cut}}
\newcommand{\Es}{\mathcal{E}}

\renewcommand{\S}{S}

\newcommand{\F}{F}
\newcommand{\f}{f}

\newcommand{\Rel}{R}
\newcommand{\rel}{r}

\newcommand{\Div}{D}
\renewcommand{\div}{d}

\newcommand{\z}{\v z}

\newcommand{\G}{G}

\newcommand{\xb}{\mathbf{x}}

\newcommand{\yb}{\mathbf{y}}

\newcommand{\wb}{\mathbf{w}} 

\newcommand{\oracle}{oracle\xspace}

\newcommand{\train}{\texttt{train}\xspace}
\newcommand{\val}{\texttt{val}\xspace}
\newcommand{\test}{\texttt{test}\xspace}

\newtheorem{lemma}{Lemma}

\newcommand{\fmin}{F_{\min{}}}

\newcommand{\Shat}{\widehat{S}}

\DeclareMathOperator{\ham}{Ham}

\title{Submodular meets Structured: Finding Diverse Subsets in Exponentially-Large Structured Item Sets}
\author{
Adarsh Prasad\\
UT Austin \\
\texttt{adarsh@cs.utexas.edu} \\
\And
Stefanie Jegelka \\
 UC Berkeley\\
\texttt{stefje@eecs.berkeley.edu} \\
\And
Dhruv Batra \\
Virginia Tech \\
\texttt{dbatra@vt.edu}
}

\nipsfinalcopy 

\begin{document} 
\maketitle


\vspace{\abstractReduceTop}

\vspace{\abstractReduceTop}
\begin{abstract} 
\vspace{\abstractReduceBot}
To cope with the high level of ambiguity faced in domains such as Computer Vision or Natural Language processing, 
robust prediction methods often search for a \emph{diverse set} of high-quality candidate solutions or \emph{proposals}. 
%
In structured prediction problems, this becomes a daunting task, as the solution space
(image labelings, sentence parses, \etc) is exponentially large. 
We study greedy algorithms for finding a diverse subset of solutions in structured-output spaces  
by drawing new connections between submodular functions \emph{over combinatorial item sets} 
and High-Order Potentials (HOPs) studied for graphical models. 
Specifically, we show via examples that when 
marginal gains of submodular diversity functions allow structured representations, this enables 
efficient (sub-linear time) approximate maximization by reducing the greedy augmentation step 
to inference in a factor graph with appropriately constructed HOPs.
We discuss benefits, trade-offs, 
and show that our constructions lead to 
significantly better proposals. 
\end{abstract} 
\vspace{\abstractReduceBot}


\vspace{\sectionReduceTop}
\section{Introduction}
\vspace{\sectionReduceBot}

Many problems in Computer Vision, Natural Language Processing and Computational Biology
involve 
mappings from an input space $\mathcal{X}$ to an exponentially large space $\calY$ of \emph{structured outputs}.  
For instance, $\calY$ may be the space of all segmentations of an image with $n$ pixels, each of which may take $L$ 
labels, so $|\calY| = L^n$.
Formulations such as Conditional Random Fields (CRFs)~\cite{lafferty_icml01_crf}, 
Max-Margin Markov Networks (M$^3$N)~\cite{taskar_nips03}, 
and Structured Support Vector Machines (SSVMs)~\cite{tsochantaridis_jmlr05}
have successfully provided principled ways of scoring all solutions $\yb \in \calY$ and 
predicting the \emph{single} highest scoring 
or maximum \emph{a posteriori} (MAP) configuration, by exploiting 
the factorization of a structured output into its constituent ``parts''.

In a number of scenarios, the posterior $\pr(\yb | \xb)$ has several modes due to ambiguities, 
and we seek not only a single best prediction but a 
\emph{set} of good predictions:\\ 
\begin{inparaenum}[(1)]
\item \textbf{Interactive Machine Learning.} Systems like Google Translate (for machine translation) 
or Photoshop (for interactive image segmentation) solve structured prediction problems that are often 
ambiguous ("what did the user really mean?"). 
Generating a small set of relevant candidate solutions for the user to select from can greatly improve the results.
\\
\item \textbf{M-Best hypotheses in cascades.} Machine learning algorithms are often cascaded, with the output of one model 
being fed into another~\citep{viola_ijcv04}.
Hence, at the initial stages it is not necessary to make 
a single perfect prediction. We rather seek a set of \emph{plausible} predictions that are subsequently 
re-ranked, combined or processed by a more sophisticated mechanism. \\
\end{inparaenum} 
In both scenarios, we ideally want a small set of $M$ \emph{plausible} (\ie, high scoring) but \emph{non-redundant} 
(\ie, diverse) structured-outputs to hedge our bets. 

\textbf{Submodular Maximization and Diversity.} 
The task of searching for a diverse high-quality subset of items from a ground set $V$ 
has been well-studied in information retrieval~\cite{carbonell_sigir98}, 
sensor placement~\cite{krause_jmlr08}, 
document summarization~\cite{lin_acl11}, viral marketing~\cite{kempe_kdd03},
and robotics~\cite{dey_rss12}. 
Across these domains, \emph{submodularity} 
has emerged as an a fundamental and 
practical concept -- 
a property of functions for measuring diversity of a subset of items. 
Specifically, a set function $F: 2^V \to \mathbb{R}$ is submodular if its \emph{marginal gains},  
$\F(a | \S) \equiv \F(\S \cup a) - \F(\S)$ are decreasing, \ie $\F(a | \S) \ge \F(a | T)$ for all 
$\S \subseteq T$ and $a \notin T$.
In addition, if $\F$ is \emph{monotone}, \ie, $\F(\S) \leq \F(T),\,\, \forall \S \subseteq T$, 
then a simple greedy algorithm 
(that in each iteration $t$ 
adds to the current set $\S^t$ 
the item with the largest marginal gain $\F(a | \S^t)$)   
achieves an
approximation factor 
of $(1-\frac{1}{e})$ \cite{nemhauser_mp78}. 
This result has had significant practical impact \cite{tutorial}. Unfortunately, if the number of items $|V|$ 
is exponentially large, then \emph{even a single linear scan for greedy augmentation is infeasible}.

In this work, we study conditions under which it is feasible to greedily maximize a submodular function 
over an exponentially large ground set $V = \{v_1, \ldots, v_N\}$ whose 
elements are \emph{combinatorial objects}, i.e., 
labelings of a \emph{base set} of $n$ variables $\yb = \{y_1, y_2, \ldots, y_n\}$.  
For instance, in image segmentation, the base variables $y_i$ are pixel labels, 
and each item $a \in V$ is a particular labeling of the pixels. 
Or, if each base variable $y_e$ indicates the presence or absence of an 
edge $e$ in a graph, then each item may represent a spanning tree or a maximal matching.  
Our goal is to find a set of $M$ plausible and diverse configurations \emph{efficiently}, 
\ie in time sub-linear in $|V|$ (ideally scaling as a low-order polynomial in $\log |V|$). 
We will assume $F(\cdot)$ to be monotone submodular, nonnegative and normalized ($F(\emptyset)=0$), and
base our study on the greedy algorithm.
As a running example, we focus on pixel labeling, where each 
base variable takes values in a set $[L] = \{1,\ldots, L\}$ of labels.

\begin{figure}[t]
\centering
\subfigure[Image]{
\frame{\includegraphics[height=0.155\linewidth]{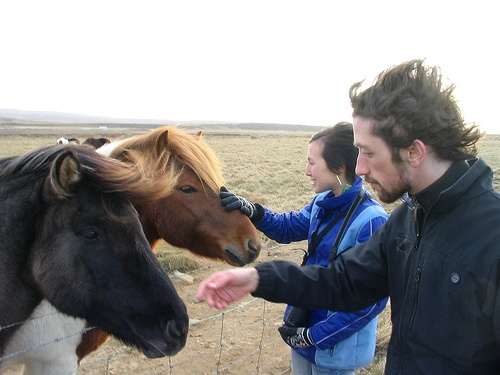}}
} \hspace{-5pt} 
\subfigure[All segmentations: $|V| = L^n$]
{\includegraphics[height=0.175\linewidth]{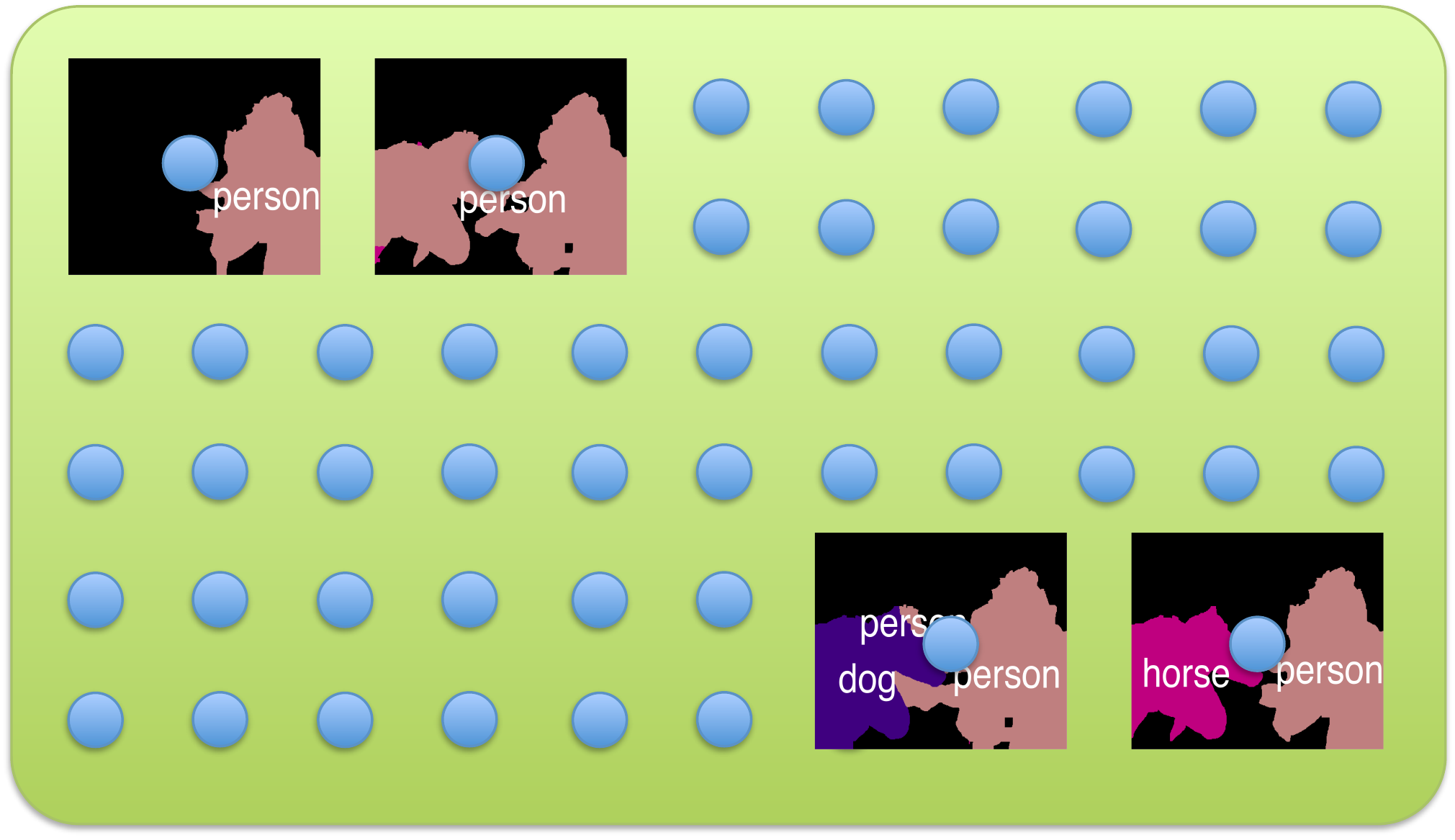}} \hspace{-2pt}
\subfigure[\small Structured Representation.]{
\includegraphics[height=0.19\linewidth]{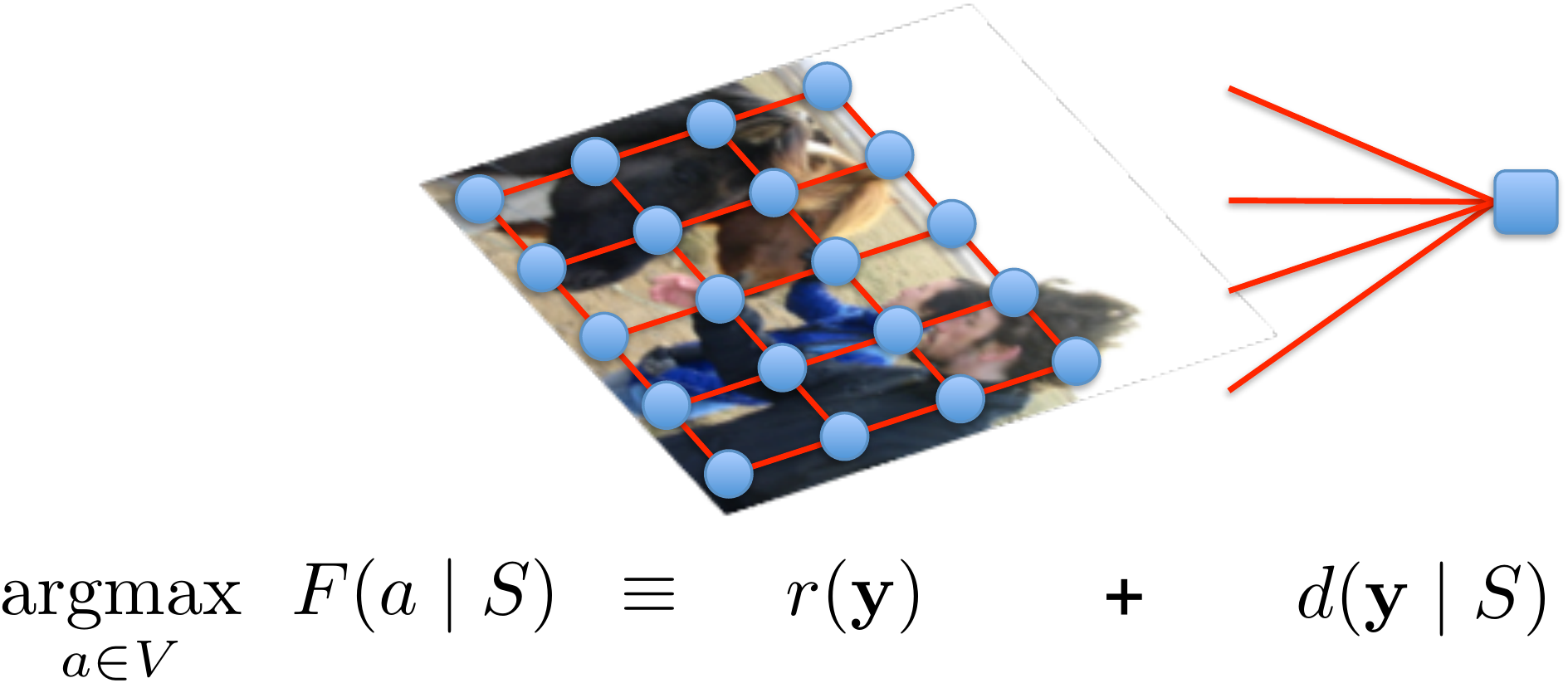}
} 
\caption{
(a) input image; 
(b) space of all possible object segmentations / labelings (each item is a segmentation); 
(c) we convert the problem of finding 
the item with the highest marginal gain $\F(a | S)$ to a MAP inference problem 
in a factor graph over base variables $\yb$ with an appropriately defined HOP.}
\label{fig:teaser}
\vspace{-12pt}
\end{figure}

\textbf{Contributions.} 
Our principal contribution is a conceptual one. 
We observe that 
\emph{marginal gains of a number of submodular functions allow structured representations}, 
and this enables efficient greedy maximization over exponentially large ground sets
-- by reducing the greedy augmentation step to a MAP inference 
query in a discrete factor graph augmented with a suitably constructed 
\emph{High-Order Potential} (HOP). 
Thus, our work draws new connections between two seemingly disparate but highly related 
areas in machine learning -- submodular maximization and inference in graphical models with structured HOPs. 
As specific examples, we construct submodular functions for three different, task-dependent 
definitions of diversity, 
and provide reductions to three different HOPs 
for which efficient inference techniques have already been developed. 
Moreover, we present a generic recipe for constructing such 
submodular functions, which may be ``plugged'' with efficient HOPs discovered in future work. 
Our empirical contribution is an efficient algorithm for producing a set of 
image segmentations with significantly higher \emph{oracle accuracy}\footnote{The accuracy 
of the most accurate segmentation in the set.} than previous works. The algorithm is general enough to 
transfer to other applications.
\figref{fig:teaser} shows an overview of our approach. 

\textbf{Related work: generating multiple solutions.}
Determinental Point Processes
 are an elegant probabilistic 
model over sets of items with a preference for diversity. 
Its generalization to a structured setting 
\cite{kulesza_nips10} assumes 
a tree-structured model, an assumption that we do not make.
%
%
Guzman-Rivera \etal \cite{rivera_nips12, rivera_aistats14} learn a set of $M$ models, 
each producing one solution, to form the set of solutions. 
Their approach requires access to the learning sub-routine and repeated re-training of the models, 
which is not always possible, as it may be expensive or proprietary. 
We assume to be given a single (pre-trained) model from which we must generate multiple diverse, good solutions.
Perhaps the closest to our setting 
are recent techniques 
for finding diverse $M$-best solutions~\cite{batra_eccv12, park_iccv11} 
or modes~\cite{chen_aistats13, chen_nips14} in graphical models. 
While \cite{chen_aistats13} and \cite{chen_nips14} are inapplicable since they are 
restricted to chain and tree graphs, 
we compare to 
other baselines in \secref{sec:hammingball} and \ref{sec:exp}. 


\vspace{\sectionReduceTop}
\subsection{Preliminaries and Notation} 
\label{sec:setup_basics}
\vspace{\sectionReduceBot}

We select from a ground set $V$ of $N$ items. Each item is a labeling $\yb = \{y_1,y_2, \ldots, y_n\}$ of $n$ base variables. 
For clarity, we use non-bold letters $a \in V$ for items, and boldface letters $\yb$ for base set configurations. 
Uppercase letters refer to functions over the ground set items $\F(a | A), \Rel(a | A), \Div(a | A)$, 
and lowercase letters to functions over base variables $\f(\yb)$, $\rel(\yb), \div(\yb)$. 
Formally, there is
a bijection $\phi: V \mapsto [L]^m$ 
that maps items $a \in V$ to their representation as base variable labelings $\yb = \phi(a)$. 
For notational simplicity, 
we often use $\yb \in \S$ to mean 
$\phi^{-1}(\yb) \in \S$, \ie the item corresponding to the labeling $\yb$ is present in the set $\S\subseteq V$. 
We write $\ell \in \yb$ if the label $\ell$ is used in $\yb$, \ie $\exists j$ s.t. $y_j = \ell$. 
For a set $c \subseteq [n]$, we use $y_c$ to denote the tuple $\{y_i \mid i \in c\}$. 
Our goal to find an ordered set or \emph{list} of items $\S \subseteq V$ that maximizes a scoring function $F$.
Lists generalize the notation of sets, and allow for reasoning of item order and repetitions. 
More details about list vs set prediction can be found in 
\cite{streeter_nips08, dey_rss12}.

\textbf{Scoring Function.}
We trade off the relevance and diversity of list $\S \subseteq V$ via a scoring function $F:2^V \to \mathbb{R}$ of the form
\begin{equation}
  \label{eq:1}
  \F(\S) = \Rel(\S) + \lambda \Div(\S),
\end{equation}
where $\Rel(\S) = \sum_{a \in \S}\Rel(a)$ is a modular nonnegative relevance function 
that aggregates the quality of all items in the list; 
$\Div(\S)$ is a monotone normalized submodular function that measure the diversity of items in $S$; 
and $\lambda \ge 0$ is a trade-off parameter. Similar objective functions were used e.g.\ in \cite{lin_acl11}. 
They are reminiscent of the general paradigm 
in machine learning 
of combining a loss 
function 
that measures quality (\eg training error) and 
a regularization term that encourages
 desirable properties (\eg smoothness, sparsity, or ``diversity''). 

\textbf{Submodular Maximization.} We aim to find a list $S$ that maximizes
%
$
  \F(\S)
$ subject to a cardinality constraint $|\S| \leq M$.
For monotone submodular $F$, this may be done via a greedy algorithm
that starts out with 
$\S^0 = \emptyset$, and iteratively adds the next best item:
\begin{equation}
  \label{eq:3}
  \S^{t} = \S^{t-1} \union a^{t}, \qquad a^{t} \in \argmax\nolimits_{a \in V} F(a \mid \S^{t-1}).
\end{equation}
The final solution $\S^M$ is within a factor of $(1-\tfrac{1}{e})$ of the optimal solution $\S^*$: 
$F(\S^M) \geq (1-\frac{1}{e})F(\S^*)$ \cite{nemhauser_mp78}. 
The computational bottleneck is that in each iteration, we must find the item with the largest marginal gain. 
Clearly, if $|V|$ has exponential size, we cannot touch each item even once. 
Instead, we propose ``augmentation sub-routines'' that exploit the structure of $V$ and 
maximize the marginal gain by solving an optimization problem over the base variables.

\vspace{\sectionReduceTop}
\section{Marginal Gains in Configuration Space}
\vspace{\sectionReduceBot}
To solve the greedy augmentation step via optimization over $\yb$, we transfer the
marginal gain from the world of items to the world of base variables and derive functions on $\yb$ from $F$: 
\begin{align}
\label{eq:definition_g}
\underbrace{F(\phi^{-1}(\yb) \mid \S)}_{f(\yb \mid \S)} = \underbrace{\Rel(\phi^{-1}(\yb))}_{\rel(\yb)} 
+ \lambda \underbrace{\Div(\phi^{-1}(\yb) \mid \S)}_{\div(\yb \mid \S)}.\\[-16pt]\nonumber
\end{align}
Maximizing $F(a|S)$ now means maximizing $\f(\yb|S)$ for $\yb = \phi(a)$. This 
can be a hard combinatorial optimization problem in general. However, as we will see, there is a 
broad class of useful functions $F$ for which $\f$ inherits exploitable structure, and  
$\argmax_{\yb} \f(\yb|S)$ can be solved efficiently, exactly or at least approximately.

\textbf{Relevance Function.} 
We use a structured relevance function $\Rel(a)$ that is the score of a factor graph 
defined over the base variables $\yb$. 
Let $G = (\calV, \calE)$ be a graph defined 
over $\{y_1,y_2,\ldots,y_n\}$, \ie $\calV= [n]$, $\calE \subseteq \binom{\calV}{2}$. 
Let $\bm{\calC} = \{C \mid C \subseteq \calV\}$ be a set of cliques in the graph, and 
 let $\theta_C: [L]^{|C|} \mapsto \mathbb{R}$ be the log-potential functions (or factors) for these cliques. The quality of an item $a = \phi^{-1}(\yb)$ is then given by
$\Rel(a) = \rel(\yb) = \sum_{C \in \bm{\calC}} \theta_C(y_C)$. 
For instance, with only node and edge factors, this quality becomes 
$\rel(\yb) = \sum_{p \in \cal V} \theta_p(y_p) + \sum_{(p,q) \in \cal E} \theta_{pq}(y_p,y_q).$
In this model, finding the \emph{single} highest quality item 
corresponds to maximum a posteriori (MAP) inference in the factor graph.

Although we refer to terms with probabilistic interpretations such as ``MAP'', 
we treat our relevance function as output of an \emph{energy-based model}~\cite{lecun_tutorial06} 
such as a Structured SVM~\cite{tsochantaridis_jmlr05}. For instance, 
$\rel(\yb) = \sum_{C \in \bm{\calC}} \theta_C(y_C) = \wb\trn \psi(\yb)$ for parameters $\wb$ and feature vector $\psi(\yb)$.
Moreover, we assume that the relevance function $\rel(\yb)$ is nonnegative%
\footnote{Strictly speaking, 
this condition is sufficient but not necessary. We only need nonnegative marginal gains.}. This assumption ensures that 
$F(\cdot)$ is monotone. If $F$ is non-monotone, algorithms other than the greedy are needed~\cite{buchbinder12,feige_focs07}.
%
We leave this generalization for future work.
%
In most application domains the relevance function is learned from data and thus 
our positivity assumption is not restrictive -- \emph{one can simply learn a positive relevance function}. 
For instance, in SSVMs, the relevance weights are learnt to maximize the margin between 
the correct labeling and all incorrect ones. 
We show in the supplement that 
SSVM parameters that assign nonnegative scores to all labelings achieve exactly the same hinge loss 
(and thus the same generalization error) as without the nonnegativity constraint.



\vspace{\sectionReduceTop}
\section{Structured Diversity Functions}
\label{sec:div}
\vspace{\sectionReduceBot}

\begin{figure}
\centering
\begin{minipage}[c]{.69\linewidth}
\subfigure[Label Groups]{
\includegraphics[width=0.5\linewidth]{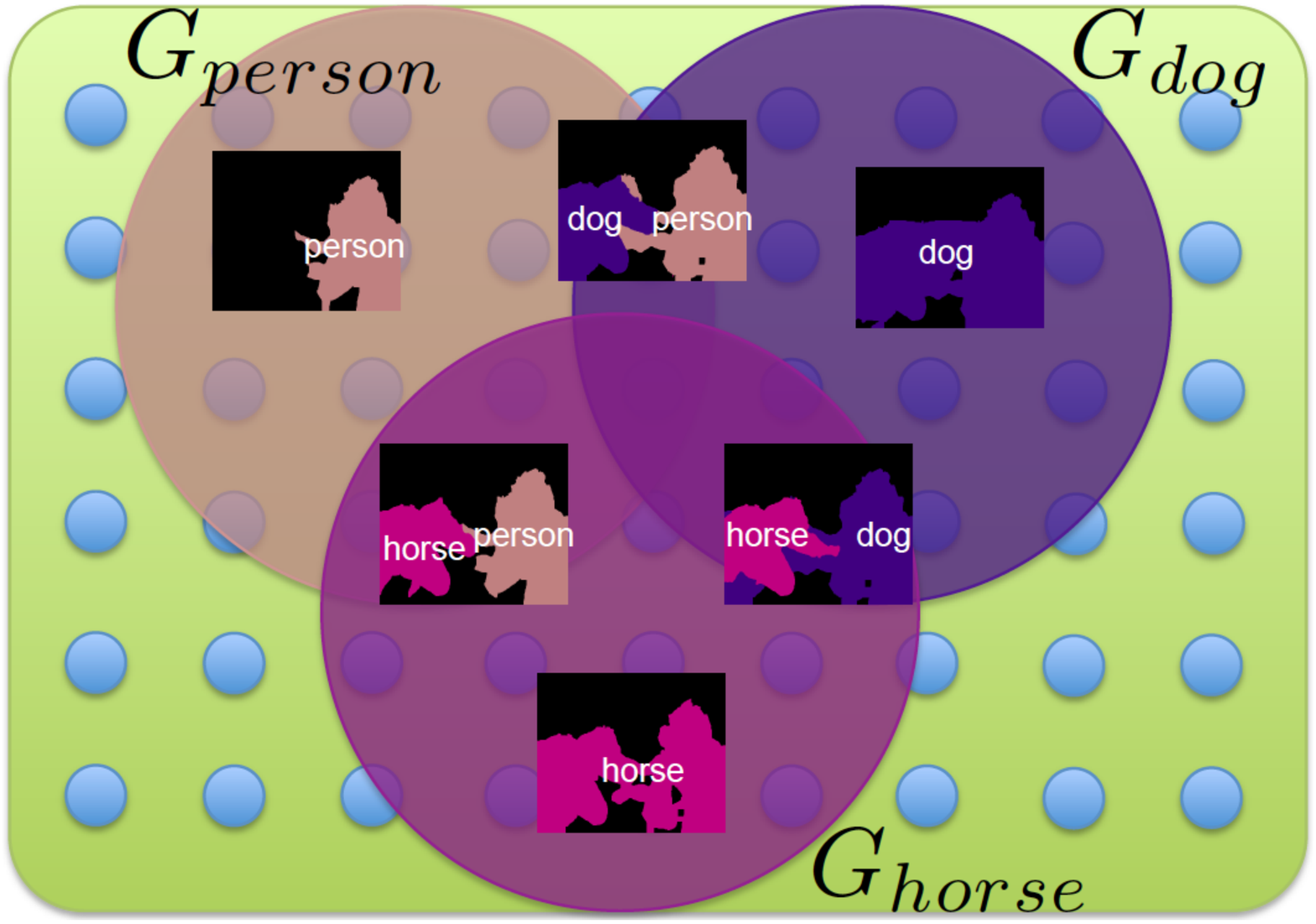}} \hspace{-5pt}
\subfigure[Hamming Ball Groups]{
\includegraphics[width=0.5\linewidth]{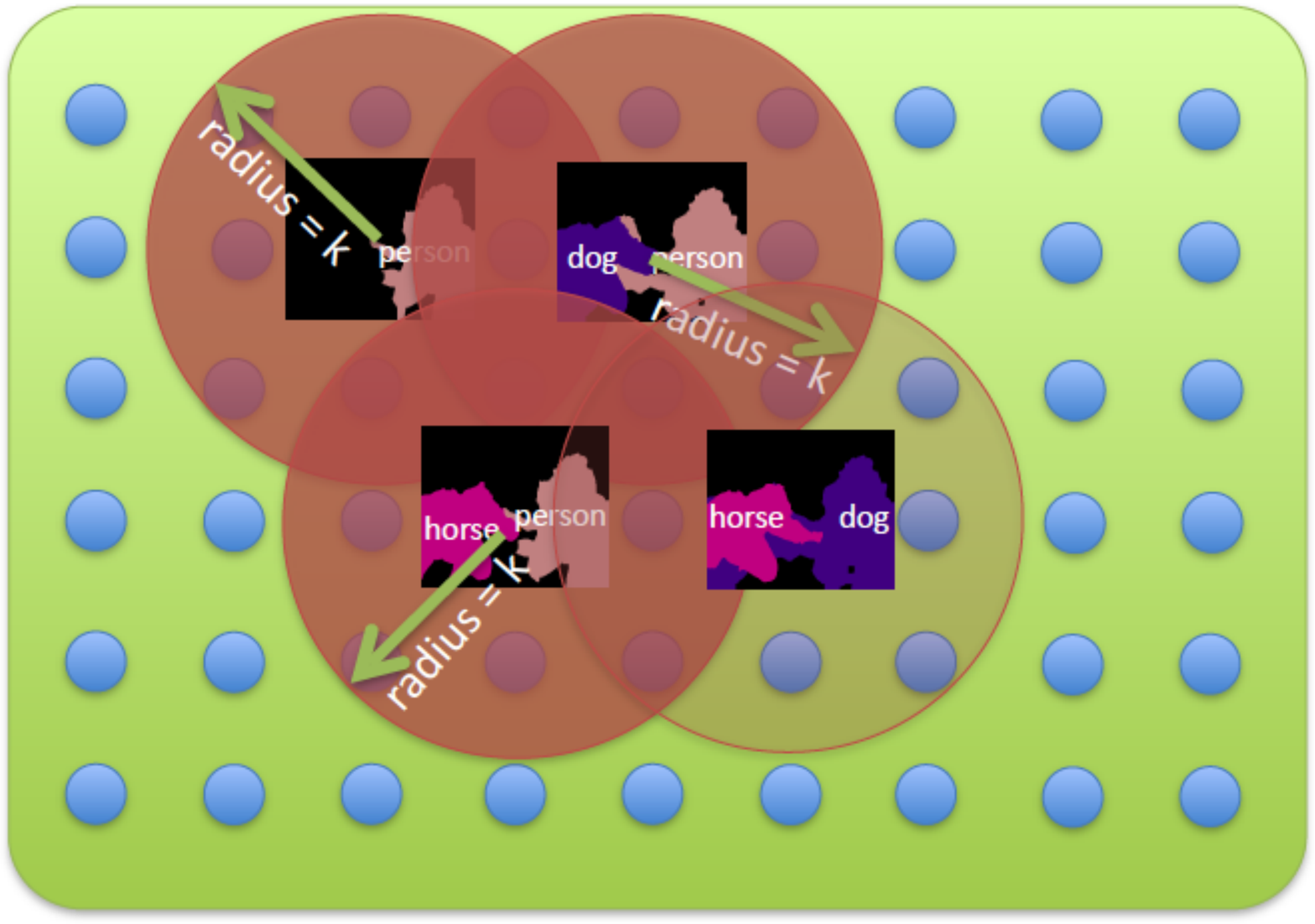}
}
\end{minipage} \hspace{5pt}
\begin{minipage}[c]{.28\linewidth}
\caption{Diversity via groups: (a) groups defined by the presence of labels (\ie \#groups = $L$); 
(b) groups defined by Hamming balls around each item/labeling (\ie \#groups = $L^n$). In each case, 
diversity is measured by how many groups are covered by a new item. See text for details. 
}
\label{fig:groups}
\end{minipage}
\vspace{-15pt}
\end{figure}

We next discuss a general recipe for constructing monotone submodular diversity functions $\Div(\S)$, 
and for reducing their marginal gains to structured representations over the base variables $\div(\yb | \S)$. 
%
Our scheme relies on constructing \emph{groups} $G_i$ that cover the ground set, \ie 
$V = \bigcup_i \G_i$. 
These groups will be defined by task-dependent characteristics -- 
for instance, in image segmentation, $\G_\ell$ can be the set of all segmentations that contain label $\ell$. 
The groups can be overlapping. 
For instance, if a segmentation $\yb$ contains pixels labeled ``grass'' and ``cow'', then $\yb \in \G_{\text{grass}}$ and 
$\yb \in \G_{\text{cow}}$. 

\textbf{Group Coverage: Count Diversity.}
Given $V$ and a set of groups $\{G_i\}$, we measure the diversity of a list $\S$
in terms of its \emph{group coverage}, \ie, the number of 
groups covered jointly by items in $\S$:
\begin{equation}
  \label{eq:threshdiv1}
  \Div(\S) = \Big| \big\{ i \mid G_i \inter \S \neq \emptyset \big\}\Big|,
\end{equation}
where we define $G_i \inter \S$ as the intersection of $G_i$ with the set of unique items in $\S$.
It is easy to show that this function is monotone submodular. 
If $\G_\ell$ is the group of all segmentations that contain label $\ell$, then 
the diversity measure of a list of segmentations $\S$ is the number of object labels that appear in any 
$a \in \S$. 
The marginal gain is the number of new groups covered by $a$: 
%
%
%
\begin{align}
  \label{eq:threshmarggain1}
  \Div(a\mid \S) =\ &\Big| \big\{i \mid a \in G_i \text{ and } \S \inter G_i = \emptyset \big\} \Big|.
\end{align}
Thus, the greedy algorithm will try to find an item/segmentation that belongs to 
as many as yet unused groups as possible. 

\textbf{Group Coverage: General Diversity.}
More generally, instead of simply counting the number of groups covered by $\S$, we 
can use a more refined decay
\begin{equation}
  \label{eq:gendiv1}
  \Div(\S) = \sum\nolimits_i h\big( \big|G_i \inter \S \big| \big).
\end{equation}
%
%
where $h$ is any nonnegative nondecreasing concave scalar function. 
This is a sum of submodular functions and hence submodular.
Eqn.~\eqref{eq:threshdiv1} is a special case of Eqn.~\eqref{eq:gendiv1} 
with 
$h(y) = \min\{1,y\}$. Other possibilities are $\sqrt{\cdot}$, or $\log(1 + \cdot)$.
For this general definition of diversity, the marginal gain is 
%
%
\begin{align}
  \label{eq:gendivmarggain}
  \Div(a \mid \S)  = & \sum\nolimits_{i: G_i \ni a }  \bigg[ h\big( 1 + \big|G_i \inter \S \big| \big) - h\big(\big|G_i \inter \S \big|\big) \bigg].
\end{align}
Since $h$ is concave, the gain 
$h\big( 1 + \big|G_i \inter \S \big| \big) - h\big(\big|G_i \inter \S \big|\big)$
decreases as $\S$ becomes larger. 
Thus, the marginal gain of an item $a$ is proportional to how rare each group $G_i \ni a$ 
is in the list $\S$. 

In each step of the greedy algorithm, we maximize 
$\rel(\yb) + \lambda \div(\yb | \S)$. 
We already established a structured representation of $\rel(\yb)$ via a factor graph on $\yb$. 
In the next few subsections, we specify three example definitions of groups $G_i$ that 
instantiate three diversity functions $\Div(\S)$. For each $\Div(\S)$, 
 we show how the
marginal gains $\Div(a | \S)$ can be expressed as a specific High-Order Potential (HOP) 
$\div(\yb | \S)$ in the factor graph over $\yb$. 
These HOPs are known to be efficiently optimizable, and hence we can solve the augmentation step efficiently.
Table \ref{table:overview} summarizes these connections. 

\textbf{Diversity and Parsimony.} 
If the groups $G_i$ are overlapping, 
some $\yb$ can belong to many groups simultaneously. While such a $\yb$ may offer an immediate large gain in diversity, in many applications it is more natural to seek a small list of \emph{complementary} labelings rather than having all labels occur in the same $\yb$.
For instance, in image segmentation with groups defined by label presence  
(Sec.~\ref{sec:labelcost}), natural scenes are unlikely to contain many labels at the same time. 
Instead, the labels should be spread \emph{across} the selected labelings $\yb \in S$.  
Hence, we include a \emph{parsimony factor} $p(\yb)$ 
that biases towards simpler labelings $\yb$. 
This term is 
a modular function and does not affect the diversity functions directly. 
We next outline some example instantiations of the functions~\eqref{eq:threshdiv1} and \eqref{eq:gendiv1}.

\vspace{\subsectionReduceTop}
\subsection{Diversity of Labels}
\label{sec:labelcost}
\vspace{\subsectionReduceBot}

For the first example, let
$G_\ell$ be the set of all 
labelings 
$\yb$ containing the label $\ell$, \ie 
$\yb \in G_\ell$ if and only if $y_j=\ell$ for some $j \in [n]$.
Such a diversity function arises in multi-class image segmentation -- if the highest scoring 
segmentation contains ``sky'' and ``grass'', then we would like to add complementary segmentations that contain 
an unused class label, say ``sheep'' or ``cow''. 

\textbf{Structured Representation of Marginal Gains.}
The marginal gain for this diversity function turns out to be a HOP called \emph{label cost}~\cite{delong_cvpr10}. It penalizes each label that occurs in a previous segmentation. 
%
Let $\texttt{lcount}_\S(\ell)$ be the number of segmentations in $\S$ that contain label $\ell$. In the simplest case of 
coverage diversity \eqref{eq:threshdiv1}, the marginal gain provides a constant reward for every as yet unseen label $\ell$: 
%
\begin{align}
\div(\yb \mid \S) = \Big| \big\{\ell \mid \yb \in G_\ell, \S \inter G_\ell = \emptyset \big\} \Big| 
= \sum_{\ell \in \yb,  \texttt{lcount}_\S(\ell) = 0} 1. 
\end{align}
%
For the general group coverage diversity \eqref{eq:gendiv1}, the gain becomes
\begin{small}
\begin{align}
\nonumber
\div(\yb | \S)  = \sum_{\ell: G_\ell \ni \yb} \bigg[  h\big( 1 + \big|G_\ell \inter \S \big| \big) - h\big(\big|G_\ell \inter \S \big|\big) \bigg]
 \label{eq:lcost_final}
&=  \sum_{\ell \in \yb} \bigg[ h\big( 1 +  \texttt{lcount}_\S(\ell) \big) - h\big( \texttt{lcount}_\S(\ell) \big) \bigg].  
\end{align}
\label{eq:lcost}
\end{small}
%
Thus, $d(\yb | S)$ rewards the presence of a label $\ell$ in $\yb$ by an amount proportional to 
how rare $\ell$ is in the segmentations already chosen in $\S$.
The parsimony factor in this setting is $p(\yb) = \sum_{\ell \in \yb} c(\ell)$. In the simplest case, $c(\ell) = -1$, \ie we are charged a 
constant for every label used in $\yb$. 

With this type of diversity (and parsimony terms), the greedy augmentation step is equivalent to performing MAP inference 
in a factor graph augmented with label reward HOPs: 
$\argmax_{\yb} \rel(\yb) + \lambda (\div(\yb \mid S) + p(\yb))$.  
\citet{delong_cvpr10} show how to perform approximate MAP inference with such label costs  
via an extension to the standard $\alpha$-expansion~\cite{Boykov2001} algorithm.



\textbf{Label Transitions.}
Label Diversity can be 
extended to reward not just the presence 
of previously unseen labels, but also the presence of previously unseen \emph{label transitions} 
(\eg, a person in front of a car or a person in front of a house). 
Formally, we define one group $\G_{\ell,\ell'}$ per label pair $\ell,\ell'$, and
$\yb \in \G_{\ell,\ell'}$ if it 
contains two adjacent variables $y_i,y_j$ with labels $y_i = \ell, y_j = \ell'$. 
This diversity function rewards the presence of a label pair $(\ell,\ell')$ by an 
amount proportional to how rare this pair is in the segmentations that are part of $\S$.
For such functions, the marginal gain $\div(\yb | \S)$ becomes a HOP 
called \emph{cooperative cuts}~\cite{jegelka_cvpr11}. 
%
%
The inference algorithm in \cite{kohli_cvpr13} gives 
a 
fully polynomial-time approximation scheme 
for any nondecreasing, nonnegative $h$, and the exact gain maximizer for the count function $h(y) = \min\{1,y\}$. Further details may be found in the supplement.

\begin{table}[t]
\centering
\footnotesize
\begin{tabular}{|lll|}
\hline
& Groups ($G_i$) & Higher Order Potentials  \\ 
\hline
Section \ref{sec:labelcost} & Labels & Label Cost  \\ 
Supplement & Label Transitions & Co-operative Cuts\\  
Section \ref{sec:hammingball} & Hamming Balls & Cardinality Potentials \\ 
\hline
\end{tabular}
\caption{\small Different diversity functions and corresponding HOPs.}
\label{table:overview}
\vspace{-20pt}
\end{table}
\vspace{\subsectionReduceTop}
\subsection{Diversity via Hamming Balls}
\label{sec:hammingball}
\vspace{\subsectionReduceBot}

The label diversity function simply rewarded the presence of a label $\ell$, irrespective of 
which or how many variables $y_i$ were assigned that label. The next diversity function 
 rewards a large \emph{Hamming distance} $\ham(\yb^1, \yb^2) = \sum_{i = 1}^{n} \ind{y_i^1 \ne y_i^2}$ 
 between configurations (where $\ind{\cdot}$ is the Iverson bracket.) 
Let $\calB_k(\yb)$ denote the k-radius Hamming ball centered at $\yb$, \ie 
$\calB(\yb) = \{\yb' \mid \ham(\yb',\yb) \le k \}$. 
The previous section constructed one group per label $\ell$. 
Now, we construct one group $G_{\yb}$ for each configuration $\yb$, 
which is the k-radius Hamming ball centered at $\yb$, \ie $G_{\yb} = \calB_k(\yb)$.   

\textbf{Structured Representation of Marginal Gains.}
For this diversity, 
the marginal gain $\div(\yb | \S)$ becomes a HOP called 
\emph{cardinality potential}~\cite{tarlow_aistats10}. 
For count group coverage, this becomes
%
\begin{subequations}
\vspace{-0pt}
\begin{align}
\div(\yb | \S) 
&= \Big| \big\{ {\yb'} \mid G_{\yb'} \inter (\S \cup \yb) \ne \emptyset \big\} \Big| - \Big| \big\{ {\yb'} \mid G_{\yb'} \inter \S \ne \emptyset \big\} \Big|  \\
\label{eq:intersection}
& = \Big| \bigcup_{\yb' \in \S \cup \yb} \calB_k(\yb') \Big| - \Big| \bigcup_{\yb' \in \S} \calB_k(\yb') \Big| 
 = \Big| \calB_k(\yb) \Big| - \bigg| \calB_k(\yb) \inter \Big[ \bigcup_{\yb' \in \S} \calB_k(\yb') \Big] \bigg|, 
\end{align}
\end{subequations}
%
\ie, the marginal gain of adding $\yb$ is the number of new configurations $\yb'$ 
covered by the Hamming ball centered at $\yb$. 
Since the size of the
intersection of $\calB_k(\yb)$ with a union of Hamming balls does not have a straightforward
structured representation, we maximize a lower bound on $d(\yb | S)$ instead: 
%
\begin{align}
\div(\yb \mid \S) \;\ge\; \div_{lb}(\yb \mid \S)\; \equiv\; \big|\calB_k(\yb) \big| - \sum\nolimits_{\yb' \in \S} \big| \calB_k(\yb) \inter \calB_k(\yb') \big| 
\label{eqn:hamlb}
\end{align}
%
This lower bound $\div_{lb}(\yb | \S)$ overcounts the intersection in Eqn.~\eqref{eq:intersection}
by summing 
the intersections with each $\calB_k(\yb')$ separately.
We can also interpret this lower bound as clipping the 
series arising from the inclusion-exclusion principle 
to the 
first-order terms. 
Importantly, \eqnref{eqn:hamlb} 
depends on $\yb$ only via its Hamming distance to $\yb'$. 
This is a 
\emph{cardinality potential} that depends only on the \emph{number} 
of variables $y_i$ assigned to a particular label. 
Specifically, ignoring constant terms, 
the lower bound can be written as a summation of cardinality factors (one for each previous 
solution $\yb' \in \S$): $\div_{lb}(\yb | \S) = \sum_{\yb' \in \S} \theta_{\yb'}(\yb)$, 
where 
$\theta_{\yb'}(\yb) = \frac{b}{|S|} - I_{\yb'}(\yb)$, $b$ is a constant (size of a k-radius Hamming ball), 
and $I_{\yb'}(\yb)$ is the number of points in the intersection 
of k-radius Hamming balls centered at $\yb'$ and $\yb$. 

With this approximation, the greedy step means performing MAP inference in a factor graph augmented with 
cardinality potentials: $\argmax_{\yb} \rel(\yb) + \lambda \div_{lb}(\yb | \S)$. This may be solved via 
message-passing, 
and 
all outgoing messages from cardinality factors can be computed in $O(n \log n)$ time \cite{tarlow_aistats10}. 
While this algorithm does not offer any approximation guarantees, it performs well in practice. 
A subtle point to note is that
$\div_{lb}(\yb|S)$ is always decreasing \wrt $|S|$ but may become negative
 due to over-counting. 
We can fix this by clamping $\div_{lb}(\yb|S)$ to be greater than $0$, but in our experiments this was unnecessary 
-- the greedy algorithm never chose a set where $\div_{lb}(\yb | S)$ was negative. 

\textbf{Comparison to DivMBest.} 
The greedy algorithm for Hamming diversity is similar in spirit to the recent work of 
\citet{batra_eccv12}, who also proposed a greedy algorithm (DivMBest) for finding diverse MAP solutions 
in graphical models.  
They did not provide any justification for greedy, and our formulation sheds some light on their work. 
Similar to our approach, at each greedy step, DivMBest involves 
maximizing a diversity-augmented score: $\argmax_{\yb} \rel(\yb) + \lambda \sum_{\yb' \in \S} \theta_{\yb'}(\yb)$. 
However, their diversity function grows \emph{linearly} with the Hamming distance, 
$\theta_{\yb'}(\yb) = \ham(\yb', \yb) = \sum_{i = 1}^{n} \ind{y_i' \ne y_i}$. 
Linear diversity rewards are not robust, and tend to over-reward 
diversity.
Our formulation uses a robust diversity function 
$\theta_{\yb'}(\yb) = \frac{b}{|S|} - I_{\yb'}(\yb)$ 
that saturates as $\yb$ moves far away from $\yb'$. 

In our experiments, we make the saturation behavior smoothly \emph{tunable} via a parameter $\gamma$:  
$I_{\yb'}(\yb) = e^{-\gamma \ham(\yb',\yb)}$. 
A larger $\gamma$ corresponds to Hamming balls of smaller radius, and 
can be set to optimize performance on validation data. We found this to work better than directly tuning 
the radius $k$.

        

\vspace{\sectionReduceTop}
\section{Experiments}
\label{sec:exp}
\vspace{\sectionReduceBot}

We apply our greedy maximization algorithms 
to two image segmentation problems: (1) interactive binary segmentation (object cutout) (\secref{sec:exp_interactive});
(2) category-level object segmentation on the PASCAL VOC 2012 dataset~\cite{pascal-voc-2012} (\secref{sec:exp_category}).
We compare all methods by their respective oracle accuracies, i.e. the accuracy of the
most accurate segmentation in the set of $M$ diverse segmentations returned by that method. For a
small value of $M \approx 5$ to $10$, a high oracle accuracy indicates that the algorithm has achieved high
recall and has identified a good pool of candidate solutions for further processing in a cascaded pipeline.
In both experiments, the label ``background'' is typically expected to appear somewhere in the image, 
and thus does not play a role in the label cost/transition diversity functions. Furthermore, 
in binary segmentation there is only one non-background label. Thus, 
we report results with Hamming diversity only (label cost and label transition diversities are not applicable). 
For the multi-class segmentation experiments, we report experiments with all three. 


\textbf{Baselines.} 
We compare our proposed methods against DivMBest~\cite{batra_eccv12}, 
which greedily produces diverse segmentation by explicitly adding a linear Hamming distance term to the factor 
graph. Each Hamming term is decomposable along the variables $y_i$ and simply modifies the 
node potentials $\tilde{\theta}(y_i) = \theta(y_i) + \lambda \sum_{\yb' \in \S} \ind{y_i \ne y'_i}$. 
DivMBest has been shown to outperform
techniques such as M-Best-MAP~\cite{yanover_nips03,batra_uai12}, which produce 
high scoring solutions without a focus on diversity, and 
sampling-based techniques, which produce diverse solutions without a focus on the relevance term~\cite{batra_eccv12}. Hence, we do not include those methods here. 
We also report results for combining different diversity functions via two 
operators: ($\otimes$), where we generate the top $\frac{M}{k}$ solutions for each of $k$ diversity functions and 
then concatenate these lists; 
and ($\oplus$), where we linearly combine diversity functions 
(with coefficients chosen by $k$-D grid search) and generate $M$ solutions using the combined diversity. 

\begin{table}[t]
\centering
\tiny
\begin{tabular}{@{}lcllllllllllll@{}}
\toprule
         & \multicolumn{3}{c}{Label Cost (LC)}            &               & \multicolumn{3}{c}{Hamming Ball (HB)} &            & \multicolumn{3}{c}{Label Transition (LT)} \\ 
\cmidrule{2-4} \cmidrule{6-8} \cmidrule{10-12}
         & \multicolumn{1}{l}{MAP}   & M=5   & M=15  &               & MAP         & M=5        & M=15       &            & MAP        & M=5        & M=15       \\
\cmidrule{2-4} \cmidrule{6-8} \cmidrule{10-12}
min$\{1,\cdot\}$ & 42.35                  & 45.43 & 45.58 &   DivMBest    & 43.43       & 51.21      & 52.90        & min$\{1,\cdot\}$ & 42.35      & 44.26      & 44.78      \\
$\sqrt(\cdot)$  & 42.35                     & 45.72 & 50.01 &   HB & 43.43   & \textbf{51.71}     & \textbf{55.32}        & $\sqrt(\cdot)$  & 42.35      & 45.43      & 46.21      \\
log$(1+\cdot)$ & \multicolumn{1}{l}{42.35} & \textbf{46.28} & \textbf{50.39} &              &             &            &              & log$(1+\cdot)$ & 42.35      & \textbf{45.92}      & \textbf{46.89}      \\ \\[-5pt]

& &  & \multicolumn{3}{c}{ \hspace{2mm}$\otimes$ Combined Diversity} & &  \multicolumn{3}{c}{\hspace{2mm}$\oplus$ Combined Diversity} &   & &  \\
\cmidrule{5-6} \cmidrule{9-9} 
&  & & & M=15  & M=16 & & & M=15 & & \\
\cmidrule{5-6} \cmidrule{9-9} 
& \multicolumn{3}{l}{HB $\otimes$ LC $\otimes$ LT} & \textbf{56.97}  & - & \multicolumn{2}{l}{DivMBest $\oplus$ HB} & 55.89 &  &  &       \\
& \multicolumn{3}{l}{ DivMBest $\otimes$ HB $\otimes$ LC $\otimes$ LT} & -   & \textbf{57.39} & \multicolumn{2}{l}{DivMBest $\oplus$ LC $\oplus$ LT} & 53.47 & &  &       \\
\bottomrule
\end{tabular}
\caption{\small PASCAL VOC 2012 \val \oracle accuracies for different diversity functions.}
\label{table:pascalresults}
\vspace{-20pt}
\end{table}

\vspace{\subsectionReduceTop}
\subsection{Interactive segmentation}
\label{sec:exp_interactive}
\vspace{\subsectionReduceBot}

In interactive foreground-background segmentation, the user provides partial labels via scribbles.
One way to minimize interactions is for the system to provide a set of candidate segmentations for the user to choose from.
We replicate the experimental setup of \citep{batra_eccv12}, who curated 
100 images from the PASCAL VOC 2012 dataset, 
and manually provided scribbles on objects contained in them.
For each image, the relevance model $\rel(\yb)$ is a 
2-label pairwise CRF, with a node term for each superpixel in the image and an
edge term for each adjacent pair of superpixels. At each superpixel, we extract
colour and texture features. We train a Transductive SVM from the partial supervision 
provided by the user scribbles. 
The node potentials are derived from the scores of these TSVMs. The edge potentials are
contrast-sensitive Potts. Fifty of the images were 
used for tuning the diversity 
parameters $\lambda, \gamma$, and the other 50 for reporting \oracle accuracies.
The 2-label contrast-sensitive Potts model results in a supermodular relevance function $\rel(\yb)$, 
which can be efficiently maximized via graph cuts \cite{Kolmogorov2004}. The Hamming ball 
diversity $\div_{lb}(\yb | S)$ is a collection of cardinality factors, which we optimize with the Cyborg implementation \cite{tarlow_aistats10}.  

\textbf{Results.} For each of the 50 test images in our dataset we generated the single best $\yb^1$ and
5 additional solutions $\{\yb^2,\ldots, \yb^6\}$ using each method. 
Table \ref{table:interactiveseg} shows the average \oracle accuracies for DivMBest, Hamming ball diversity, 
and their two combinations. We can see that the combinations slightly outperform both approaches. 
\begin{table}[h]
\centering
\scriptsize
\begin{tabular*}{1\columnwidth}{@{\extracolsep{\fill}}lllll@{}}
\toprule
                 & MAP   & M=2   & M=6   &  \\ \midrule
DivMBest         & 91.57 & 93.16 & 95.02 &  \\
Hamming Ball & 91.57 & 93.95 & 94.86 &  \\
DivMBest$\otimes$Hamming Ball & - & - & 95.16 & \\
DivMBest$\oplus$Hamming Ball & - & - & 95.14 & \\ \bottomrule
\end{tabular*}
\caption{Interactive segmentation: \oracle pixel accuracies averaged over 50 test images}
\vspace{-5pt}
\label{table:interactiveseg}
\end{table}

\vspace{\subsectionReduceTop}
\subsection{Category level Segmentation}
\label{sec:exp_category}
\vspace{\subsectionReduceBot}

In category-level object segmentation, 
we label each pixel with one of 20 object categories
or background. 
%
We construct a multi-label pairwise CRF on superpixels. Our node potentials are outputs 
of category-specific regressors trained by \cite{carreira_eccv12}, and our edge potentials are 
multi-label Potts. Inference in the presence of diversity terms is performed with the implementations 
of \citet{delong_cvpr10} for label costs, \citet{tarlow_aistats10} for Hamming ball diversity, 
and \citet{Boykov2001} for label transitions. 

\begin{figure}[htbp]
\centering
\begin{minipage}[c]{.63\linewidth}
\includegraphics[scale=0.36]{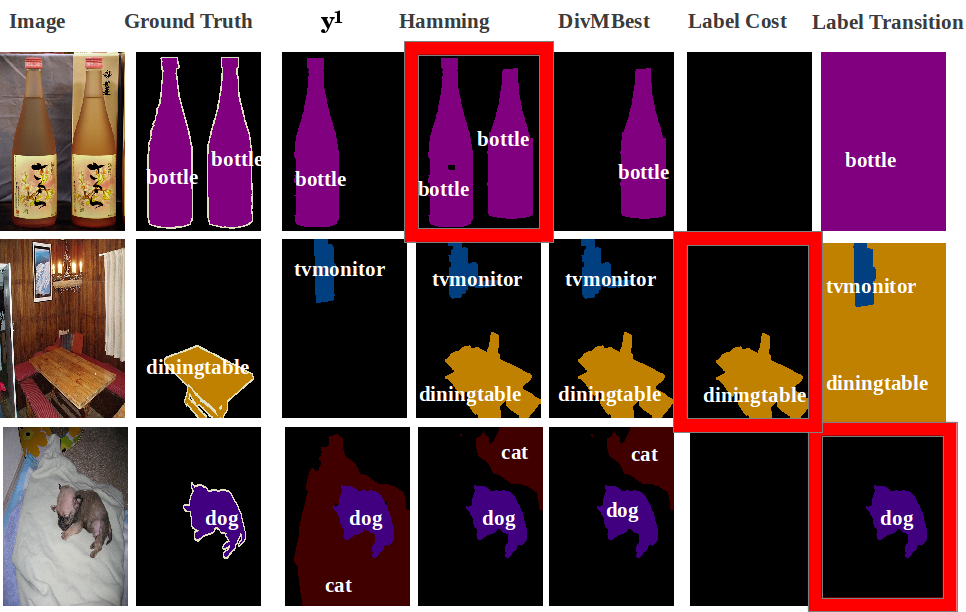}
\end{minipage}\qquad
\begin{minipage}[c]{.3\linewidth}
\caption{Qualitative Results: each row shows the original image, ground-truth segmentation (GT) from PASCAL, 
the single-best segmentation $\yb^1$, and \oracle segmentation from the $M=15$ segmentations (excluding $\yb^1$) 
for different definitions of diversity. Hamming typically performs the best. In certain situations (row3), 
label transitions help since the single-best segmentation $\yb^1$ included a rare pair of labels (dog-cat boundary). }
\label{fig:qualitative}
\end{minipage}
\end{figure}

\textbf{Results}. We evaluate all methods on the PASCAL VOC 2012 data~\cite{pascal-voc-2012}, consisting of
\train, \val and \test partitions with about 1450 images each. 
We train the regressors of  \cite{carreira_eccv12} on \train, and report \oracle accuracies of different methods on \val 
(we cannot report oracle results on \test since those annotations are not publicly available). 
Diversity parameters ($\gamma$, $\lambda$) are chosen by performing cross-val on \val.
The standard PASCAL accuracy is the corpus-level intersection-over-union 
measure, averaged over all categories. 
For both label cost and transition, we try 3 different concave 
functions $h(\cdot) = \min\{1,\cdot\}$, $\sqrt(\cdot)$ and $\log(1+\cdot)$.
Table \ref{table:pascalresults} shows the results.\footnote{MAP accuracies in Table~\ref{table:pascalresults} 
are different because of two different approximate MAP solvers: LabelCost/Transition use 
alpha-expansion and HammingBall/DivMBest use message-passing.} 
Hamming ball diversity performs the best, followed by DivMBest, and label cost/transitions are worse here. 
We found that while worst on average, label transition diversity helps in an interesting scenario -- 
when the first best segmentation $\yb^1$ includes a pair of rare or mutually confusing labels (say dog-cat). 
\figref{fig:qualitative} shows an example, and more illustrations are provided in the supplement. 
In these cases, searching for a different label transition produces a better segmentation. 
Finally, we note that lists produced with combined diversity significantly outperform 
any single method (including DivMBest). 



\vspace{\sectionReduceTop}
\section{Discussion and Conclusion}
\label{sec:discussions}
\vspace{\sectionReduceBot}

In this paper, we study greedy algorithms for maximizing scoring functions that promote diverse sets of combinatorial configurations. 
This problem arises naturally in domains such as 
Computer Vision, Natural Language Processing, or
Computational Biology, 
where we want to search for a set of \emph{diverse high-quality} solutions in a 
structured output space. 

The diversity functions we propose are monotone submodular functions by construction. Thus, 
if $\rel(\yb) + p(\yb) \geq 0$ for all $\yb$, then the entire scoring function $F$ is monotone submodular.
We showed that $\rel(\yb)$ can simply be learned to be positive. 
The greedy algorithm for maximizing monotone submodular functions has proved useful in moderately-sized \emph{unstructured} 
spaces. To the best of our knowledge, this is the first generalization to exponentially large structured output spaces. 
In particular, our contribution lies in reducing the greedy augmentation step to inference with structured, efficiently solvable 
HOPs. This insight makes new connections between submodular optimization and work on inference in graphical models.
We now address some questions. 

\textbf{Can we sample?} 
One question that may be posed is how random sampling would perform for large ground sets $V$. 
Unfortunately, the expected value of 
a random sample of $M$ elements can be much worse than the optimal value $F(\S^*)$, especially if $N$ is large. 
Lemma~\ref{lem:random} is proved in the supplement.
%
\begin{lemma}\label{lem:random}
  Let $\S \subseteq V$ be a sample of size $M$ taken uniformly at random. There exist monotone submodular functions where $\mathbb{E}[F(\S)] \leq \frac{M}{N} \max_{|\S| = M}F(\S)$.
\end{lemma}

\textbf{Guarantees?} 
If $F$ is nonnegative, monotone submodular, 
then using an exact HOP inference algorithm will clearly result in 
an approximation factor of $1-1/e$.
%
But many HOP inference procedures are approximate. Lemma~\ref{lem:approxfactor} 
formalizes how approximate inference affects the approximation bounds.
%
%
\begin{lemma} 
\label{lem:approxfactor}
Let $F \geq 0$ be monotone submodular.
If each step of the greedy algorithm uses an approximate marginal gain maximizer $b^{t+1}$ with 
$F(b^{t+1}\mid \S^t) \geq \alpha \max_{a \in V}F(a \mid \S^{t}) - \epsilon_{t+1}$, 
then $F(\S^M) \geq (1-\frac{1}{e^{\alpha}})\max_{|\S| \le M}F(\S) - \sum_{i=1}^M\epsilon_t$.
\vspace{-5pt}
\end{lemma}
Parts of Lemma~\ref{lem:approxfactor} have been observed in previous work \cite{goundan09,streeter_nips08}; 
we show the combination in the supplement. 
If $F$ is monotone but not nonnegative, then Lemma~\ref{lem:approxfactor} can be extended to a relative error bound $\frac{F(\S^M) - F_{\min}}{F(\S^*) - F_{\min}} \geq (1 - \frac{1}{e^\alpha}) - \frac{\sum_i \epsilon_i}{F(\S^*) - F_{\min}}$ that refers to $F_{\min} = \min_S F(S)$ and the optimal solution $S^*$.
While stating these results, we add that further 
additive approximation losses occur if the approximation bound for inference is computed on a 
shifted or reflected function (positive scores vs positive energies). 
We pose theoretical improvements as an open question for future work.
That said, our 
experiments convincingly show that the algorithms perform very well in practice, even when there 
are no guarantees (as with Hamming Ball diversity).

\textbf{Generalization.} In addition to the three specific examples in Section~\ref{sec:div}, our constructions generalize to the broad HOP class of 
upper-envelope potentials~\cite{kohli_cvpr10}. The details are provided in the supplement.

{\small
\textbf{Acknowledgements.} 
We thank Xiao Lin for his help. The majority of this work was done while AP was an intern at Virginia Tech. 
AP and DB were partially supported by 
the National Science Foundation under Grant No. IIS-1353694 and IIS-1350553,  
the Army Research Office YIP Award W911NF-14-1-0180, 
and the Office of Naval Research Award N00014-14-1-0679, awarded to DB. 
SJ was supported by gifts from Amazon Web Services, Google, SAP, 
The Thomas and Stacey Siebel Foundation, Apple, C3Energy, Cisco, Cloudera, 
EMC, Ericsson, Facebook, GameOnTalis, Guavus, HP, Huawei, Intel, Microsoft, NetApp, 
Pivotal, Splunk, Virdata, VMware, WANdisco, and Yahoo!.
}

\vspace{-12pt}

\begin{small}
\bibliographystyle{abbrvnat}
\bibliography{local}
\end{small}

\begin{appendix}
\section{Structured SVMs with nonnegativity constraint}
In this section, we will show that SSVMs have no natural origin, and that parameters learnt with non-negativity 
constraints achieve exactly the same hinge loss as achieved without the nonnegativity constraint. \\
For a set of $\ell$ training instances ($\boldsymbol{x}^n,\yb^n$) $\in \mathcal{X}\times\mathcal{Y}, n=1,\dots,\ell$ from a sample space
$\mathcal{X}$ and label space $\mathcal{Y}$, 
the structured SVM minimizes the following regularized risk function.
\begin{equation}
\underset{\boldsymbol{w}}{\min} \quad \|\boldsymbol{w}\|^2 + C \sum_{n=1}^{\ell}
   \underset{\yb \in\mathcal{Y}}{\max} \left(\Delta(\yb^n,\yb) + \boldsymbol{w}'\Psi(\boldsymbol{x}^n,\yb) - \boldsymbol{w}'\Psi(\boldsymbol{x}^n,\yb^n)\right)
\end{equation}
The function $\Delta: \mathcal{Y} \times \mathcal{Y} \to \mathbb{R}_+$ measures a distance in label space and is 
an arbitrary function satisfying $\Delta(\yb,\yb') \geq 0$ and  $\Delta(\yb,\yb)=0 \;\; \forall \yb,\yb' \in \mathcal{Y}$. 
The function $\Psi: \mathcal{X} \times \mathcal{Y} \to \mathbb{R}^d$ is a feature function, extracting some feature vector from a 
given sample and label. 

Because the regularized risk function above is non-differentiable, it is often 
reformulated in terms of a quadratic program by introducing one slack variable $\xi_n$ for each sample, each 
representing the value of the maximum. The standard structured SVM primal formulation is given as follows.

\begin{equation}
 \begin{array}{cl}
  \underset{\boldsymbol{w},\boldsymbol{\xi}}{\min} & \|\boldsymbol{w}\|^2 + C \sum_{n=1}^{\ell} \xi_n\\
  \textrm{s.t.} & \boldsymbol{w}' \Psi(\boldsymbol{x}^n,\yb^n) - \boldsymbol{w}' \Psi(\boldsymbol{x}^n,\yb) + \xi_n \geq \Delta(\yb^n,\yb),\qquad n=1,\dots,\ell,\quad \forall \yb \in \mathcal{Y}
  \end{array}
  \label{eq:biasedsvm}
\end{equation}

We will refer to the primal objective in Equation~\ref{eq:biasedsvm} as $\boldsymbol{P}_{\textrm{1}}$ and the optimal 
solution as $\boldsymbol{w}_{\textrm{1}},\boldsymbol{\xi_\textrm{1}}$. \\

Next, consider the following augmented formulation, where $\hat{\boldsymbol{w}} = \begin{bmatrix} w \\ b \end{bmatrix}$, $\hat{\Psi}(\boldsymbol{x}^n,\yb^n) = \begin{bmatrix} \Psi(\boldsymbol{x}^n,\yb^n) \\ 1 \end{bmatrix}$:
\begin{equation}
 \begin{array}{cl}
  \underset{\boldsymbol{\hat{w}},\boldsymbol{\xi}}{\min} & \|\boldsymbol{w}\|^2 + C \sum_{n=1}^{\ell} \xi_n\\
  \textrm{s.t.} & \hat{\boldsymbol{w}}' \hat{\Psi}(\boldsymbol{x}^n,\yb^n) - \hat{\boldsymbol{w}}' \hat{\Psi}(\boldsymbol{x}^n,\yb) + \xi_n \geq \Delta(\yb^n,\yb),\qquad n=1,\dots,\ell,\quad \forall \yb \in \mathcal{Y}
  \end{array}
  \label{eq:unbiasedsvm}
\end{equation}
We will refer to the primal objective in Equation~\ref{eq:unbiasedsvm} as $\boldsymbol{P}_{\textrm{2}}$ and the optimal 
solution as $\hat{\boldsymbol{w}_{\textrm{2}}} = \begin{bmatrix} w_2 \\ b_2 \end{bmatrix},\boldsymbol{\xi_\textrm{2}}$.
\textit{Note: We do not regularize b.}
\paragraph{Claim 1.} $\boldsymbol{P}_{\textrm{1}}(\boldsymbol{w}_{\textrm{1}},\boldsymbol{\xi_\textrm{1}}) = \boldsymbol{P}_{\textrm{2}}(\hat{\boldsymbol{w}_{\textrm{2}}},\boldsymbol{\xi_\textrm{2}}) $ \\
\textbf{Proof.} Adding a bias feature (b) doesn't affect the objective function and every constraint is invariant of it.
Hence, the two problems are equivalent.

Next, consider the formulation, where we extend $\boldsymbol{P}_{\textrm{2}}$ to enforce non-negativity of scores.
\begin{equation}
 \begin{array}{cl}
  \underset{\boldsymbol{\hat{w}},\boldsymbol{\xi}}{\min} & \|\boldsymbol{w}\|^2 + C \sum_{n=1}^{\ell} \xi_n\\
  \textrm{s.t.} & \hat{\boldsymbol{w}}' \hat{\Psi}(\boldsymbol{x}^n,\yb^n) - \hat{\boldsymbol{w}}' \hat{\Psi}(\boldsymbol{x}^n,\yb) + \xi_n \geq \Delta(\yb^n,\yb),\qquad n=1,\dots,\ell,\quad \forall \yb \in \mathcal{Y} \\
   & \hat{\boldsymbol{w}}'\hat{\Psi}(\boldsymbol{x}^n,\yb) \ge 0 \qquad n=1,\dots,\ell,\quad \forall \yb \in \mathcal{Y}
  \end{array}
  \label{eq:non-negsvm}
\end{equation}
We will refer to the primal objective in Equation~\ref{eq:non-negsvm} as $\boldsymbol{P}_{\textrm{3}}$ and the optimal 
solution as $\hat{\boldsymbol{w}_{\textrm{3}}} = \begin{bmatrix} w_3 \\ b_3 \end{bmatrix},\boldsymbol{\xi_\textrm{3}}$.

\paragraph{Claim 2.} $\boldsymbol{P}_{\textrm{2}}(\hat{\boldsymbol{w}_{\textrm{2}}},\boldsymbol{\xi_\textrm{2}}) = \boldsymbol{P}_{\textrm{3}}(\hat{\boldsymbol{w}_{\textrm{3}}},\boldsymbol{\xi_\textrm{3}}) $ \\
\textbf{Proof.} It is easy to see that $\boldsymbol{P}_{\textrm{2}}(\hat{\boldsymbol{w}_{\textrm{2}}},\boldsymbol{\xi_\textrm{2}}) \le \boldsymbol{P}_{\textrm{3}}(\hat{\boldsymbol{w}_{\textrm{3}}},\boldsymbol{\xi_\textrm{3}}) $
as $(\hat{\boldsymbol{w}_{\textrm{2}}},\boldsymbol{\xi_\textrm{2}})$ is the optimal solution for $\boldsymbol{P}_{\textrm{2}}$ and $(\hat{\boldsymbol{w}_{\textrm{3}}},\boldsymbol{\xi_\textrm{3}})$ is also a feasible solution
for $\boldsymbol{P}_{\textrm{2}}$. \\
Consider the vector $\hat{\boldsymbol{w}}_{\textrm{2}}^* = \begin{bmatrix} w_2 \\ - \underset{n}{\min} \ \underset{\yb}{\min} \ w_2'\Psi(\boldsymbol{x}^n,\yb)\end{bmatrix}$, 
Since the bias term doesn't occur in the objective function, 
hence $\boldsymbol{P}_{\textrm{3}}(\hat{\boldsymbol{w}}_{\textrm{2}}^*,\boldsymbol{\xi_\textrm{2}}) = \boldsymbol{P}_{\textrm{2}}(\hat{\boldsymbol{w}_{\textrm{2}}},\boldsymbol{\xi_\textrm{2}})$. \\

Also, $(\hat{\boldsymbol{w}}_{\textrm{2}}^*,\boldsymbol{\xi_\textrm{2}})$ is a feasible solution to $\boldsymbol{P}_{\textrm{3}}$,
hence, $\boldsymbol{P}_{\textrm{3}}(\hat{\boldsymbol{w}_{\textrm{3}}},\boldsymbol{\xi_\textrm{3}}) \le \boldsymbol{P}_{\textrm{3}}(\hat{\boldsymbol{w}}_{\textrm{2}}^*,\boldsymbol{\xi_\textrm{2}}) = \boldsymbol{P}_{\textrm{2}}(\hat{\boldsymbol{w}_{\textrm{2}}},\boldsymbol{\xi_\textrm{2}})$.\\
Therefore, $\boldsymbol{P}_{\textrm{2}}(\hat{\boldsymbol{w}_{\textrm{2}}},\boldsymbol{\xi_\textrm{2}}) = \boldsymbol{P}_{\textrm{3}}(\hat{\boldsymbol{w}_{\textrm{3}}},\boldsymbol{\xi_\textrm{3}}) $. \\

Therefore, even after adding the nonnegativity constraints, the solutions 
achieve the same values for the \emph{regularised risk function} and hence are expected to have 
the same \emph{generalization} guarantees. In practice, the non-negativity constraints can be added in a cutting-plane procedure via a MAP call.

\section{Label Transitions}
\label{sec:labeltransition}

\textbf{Groups and Motivating Scenario.}
In this section, we generalize the label cost diversity function to reward not just the presence 
of certain labels, but the presence of certain \emph{label transitions}. For instance, if the highest 
scoring segmentation contains a ``cow'' on ``grass'', this diversity function will reward other 
segmentations for containing novel label transitions, such as ``sheep-grass'' or ``cow-ground'' or ``sheep-sky''. 
Formally, we define one group $\G_{\ell,\ell'}$ per label pair $\ell,\ell'$, and an item 
$a$ belongs to $\G_{\ell,\ell'}$ if $\yb = \phi(a)$
contains two adjacent variables $y_i,y_j$ with labels $y_i = \ell, y_j = \ell'$. 

\textbf{Structured Representation of Marginal Gains.}
For diversity of label transitions, the marginal gain $\Div(a \mid \S)$ becomes a HOP 
called \emph{cooperative cuts}~\cite{jegelka_cvpr11}.
Let $\Cut_{\yb}(\ell,\ell') = \{(y_i,y_j) \in \Es \mid y_i = \ell, y_j = \ell'\}$ be the cut set for a specific label 
transition $(\ell, \ell')$. 
Further, let $\#\Cut_\S(\ell,\ell')$ count the number of items $a \in \S$ that 
contain at least one $(\ell,\ell')$ label transition: 
$\#\Cut_\S(\ell,\ell') = |\{a \in \S \mid \Cut_{\phi(a)}(\ell,\ell') \neq \emptyset\}| = |\S \inter G_{\ell,\ell'}|$. 
The marginal gains for this diversity function are: 
%
\begin{subequations}
\begin{align}
&\hspace{-5pt} \div(\yb \mid S) = \Div(\phi^{-1}(\yb) \mid \S) \\
& \hspace{-10pt} = \sum_{\ell,\ell'} \,\, h(\#\Cut_{\S \cup \phi^{-1}(\yb)}(\ell,\ell')) - h(\#\Cut_\S(\ell,\ell'))\\
& \hspace{-10pt} =  \hspace{-10pt} \sum_{\substack{\ell,\ell' \\ \Cut_{\yb}(\ell,\ell')\neq \emptyset}} \hspace{-10pt} h(1+\#\Cut_\S(\ell,\ell')) - h(\#\Cut_\S(\ell,\ell'))
\end{align}
\end{subequations}
Similar to single label groups, the gain for a label pair $(\ell,\ell')$ decreases as
$\#\Cut_\S(\ell,\ell')$ grows. 
Thus, $d(\yb \mid \S)$ rewards the presence of pair $(\ell,\ell')$ by an 
amount proportional to how rare it is in the segmentations in $\S$.
Analogously to the label costs, the parsimony factor in this setting is $p(\yb) = \sum_{\Cut_{\yb}(\ell,\ell')\neq \emptyset} c(\ell,\ell')$  encouraging each individual $\yb$ 
to have a small number of label transitions. 
Specifically, when using a count coverage and parsimony term with $c(\ell,\ell')=-1$,
we eventually maximize 
\begin{subequations}
\begin{align}
  &\rel(\yb) + p(\yb) + \div(\yb) \\
  &= \rel(\yb) - \sum_{\ell,\ell': \S \in G_{\ell,\ell'}} \min\{  \#\Cut_{\yb}(\ell,\ell'), \, 1\},
\end{align}
\end{subequations}
which is a supermodular function on the set of cut edges for each label transition. Thus, the multi-label cooperative cut inference algorithm by Kohli \etal \cite{kohli_cvpr13} applies.
%
The construction for general $h$ looks similar, with a degrading cost in front of the min.

\section{Experiments}

For the sake of completeness and to show the difference in sets of solutions generated by different diversity functions, we show sample
sets of solutions generated for a given image (Fig. \ref{fig:samplesol2}, \ref{fig:samplesol} and \ref{fig:samplesol3}). These results help in understanding the behavior of different diversity functions.

\begin{figure*}[t]
  \centering
  \includegraphics[scale=0.5]{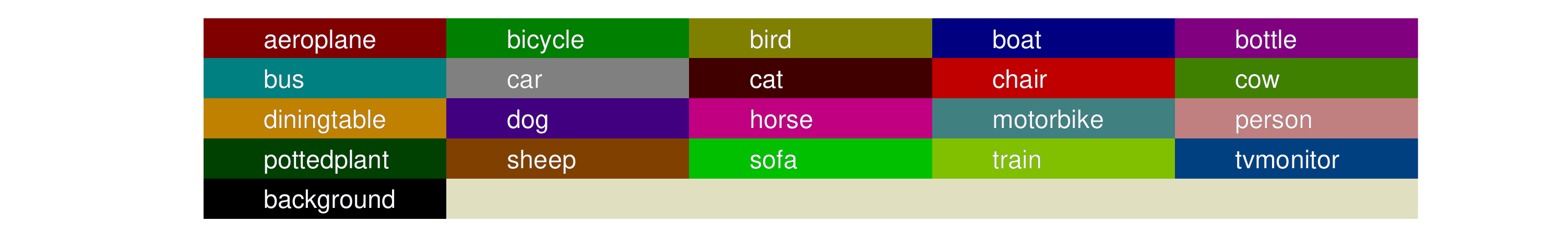}
  \caption{Color map for reading VOC segmentation results.}
  \label{fig:colormap}
\end{figure*}

\begin{figure}[htbp]
  \centering
  \includegraphics[scale=0.5]{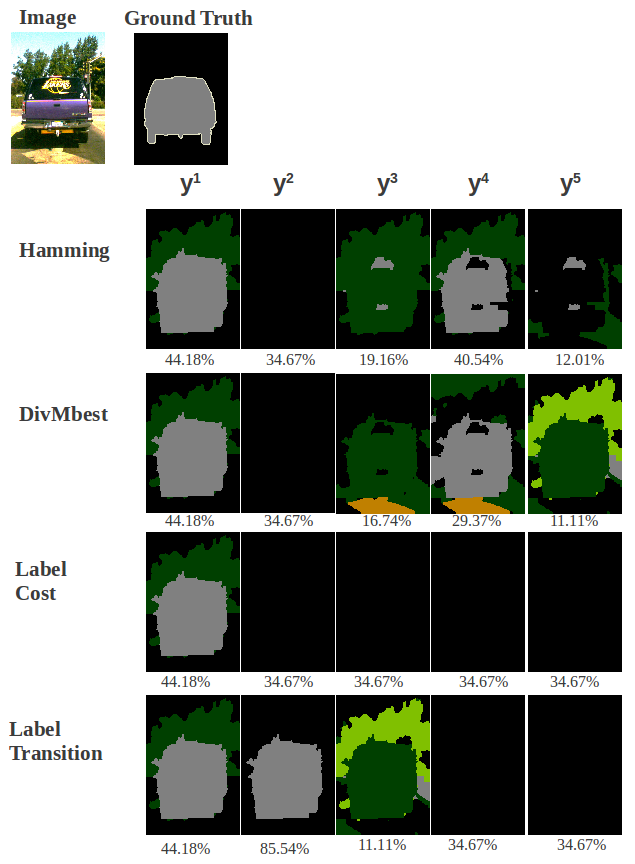}
  \caption{Sets of solutions generated with different diversity functions.}
  \label{fig:samplesol2}
\end{figure}

\begin{figure}[t]
\centering
\subfigure{\includegraphics[scale=0.44]{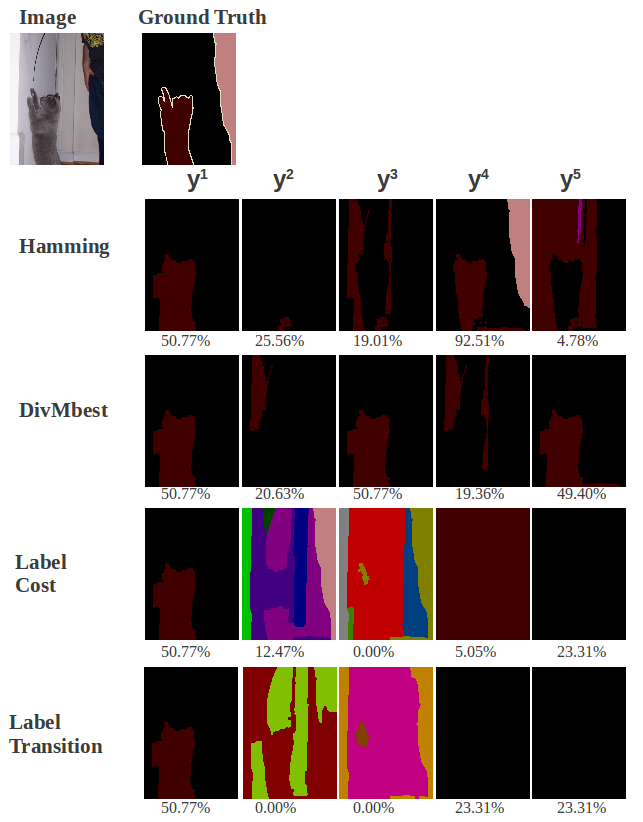}}\\
\subfigure{\includegraphics[scale=0.44]{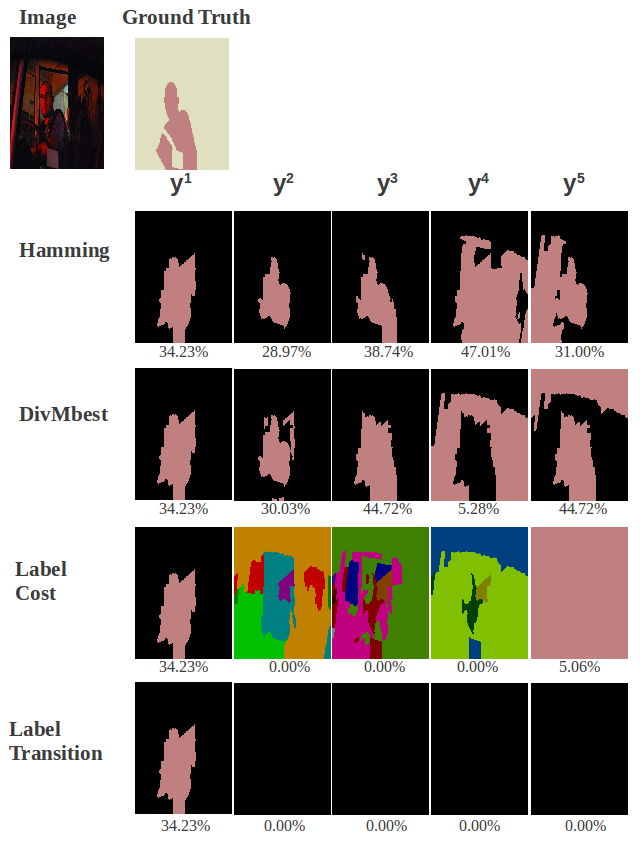}} 
\caption{Sets of solutions generated with different diversity functions.}
\label{fig:samplesol}
\end{figure}

\begin{figure}[t]
\centering
\subfigure{\includegraphics[scale=0.44]{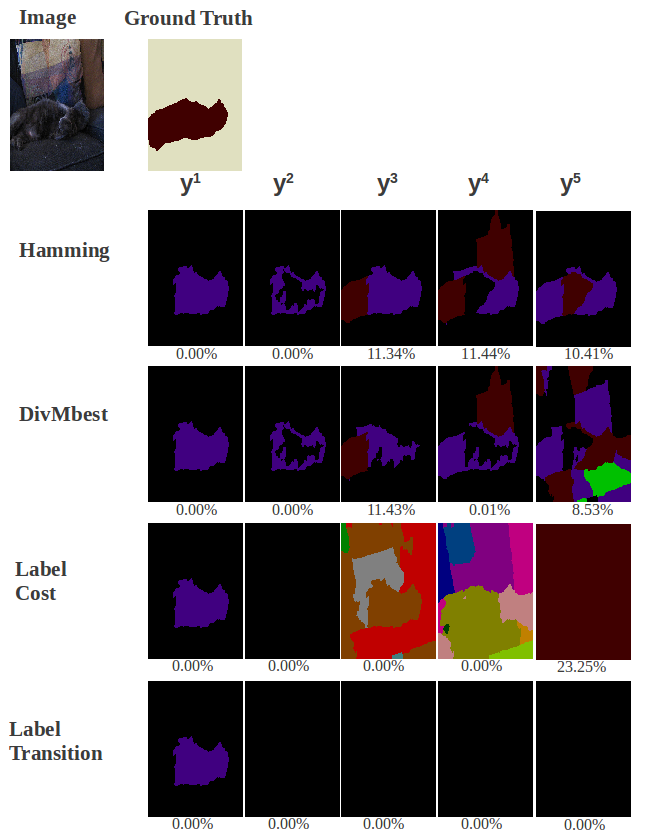}} \\
\subfigure{\includegraphics[scale=0.44]{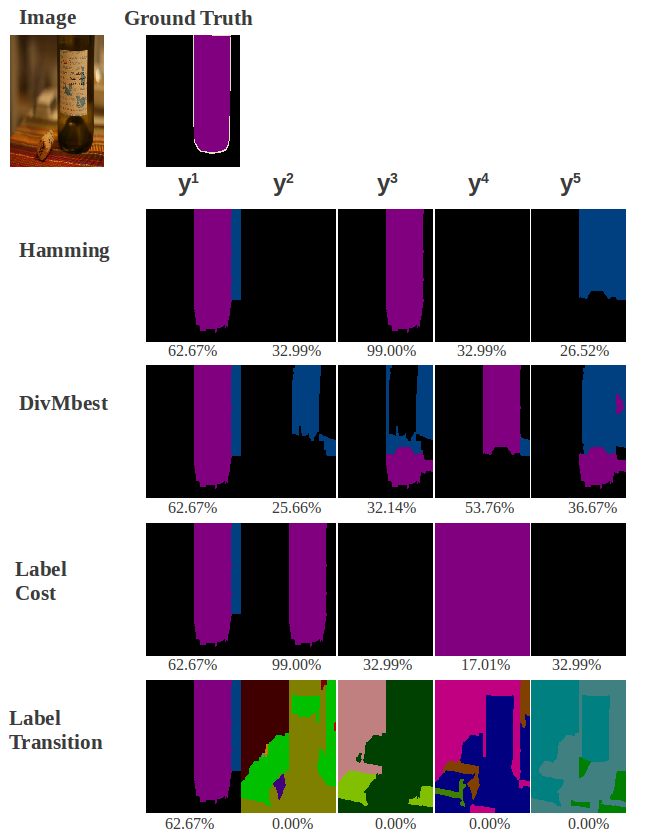}} 
\caption{Sets of solutions generated with different diversity functions.}
\label{fig:samplesol3}
\end{figure}

\newpage


\vspace{40pt}
\section{Proof of Lemma 1}

\begin{lemma}
  Let $S$ be a sample of size $M$ taken uniformly at random. There exist monotone submodular functions where $\mathbb{E}[F(S)] \leq (M/N + \epsilon/M) \max_{|S| \le M}F(S)$ for any $\epsilon \geq 0$.
\end{lemma}
The bound in the main paper follows with $\epsilon=0$.

\begin{proof}
  To prove this statement, we consider a specific worst-case function. Let $R \subseteq V$ be a fixed set of size $M$, and let
  \begin{equation}
    \label{eq:1}
    F(S) = |S \inter R| + \epsilon\min\{|S \setminus R|, 1\}.
  \end{equation}

  The function $F$ is obviously nondecreasing and it is also submodular. 
For this function, the cardinality-constrained optimum is
\begin{equation}
  \label{eq:2}
  \max_{|S| \le M}F(S) = F(R) = M.
\end{equation}
The expected value of an $M$-sized sample is the expectation of a hypergeometric distribution plus a correction taking into account that every set except $R$ will have value at least 1 (using the second part of $F$):
\begin{align}
  \mathbb{E}_S[ F(S) ] &= {N \choose M}^{-1}\left(\sum_{S \subseteq V, |S|=M} |S \inter R| + {N \choose M}\epsilon - \epsilon\right)\\
    &= {N \choose M}^{-1}\left( \sum_{r=1}^M {M \choose r}{N-M \choose M-r}r  -\epsilon \right) + \epsilon\\
    &= \frac{M^2}{N} + \epsilon - {N \choose M}^{-1}\epsilon\\
    &< (M/N + \epsilon/M) F(R).
\end{align}
\end{proof}
This is also a fairly tight bound: 
 if we sample each element with probability $M/N$, then, using Lemma 2.2 in \cite{feige07} and the monotonicity of $F$, it holds that $\mathbb{E}_S[F(S)] \geq \frac{M}{N}F(V) \geq \frac{M}{N}F(S^*)$.

\section{Proof of Lemma 2}

\begin{lemma} 
\label{lem:approxfactor}
Let $F \geq 0$ be monotone submodular.
  If each step of the greedy algorithm uses an approximate gain maximizer $b^{i+1}$ with 
  $F(b^{i+1}\mid S^i) \geq \alpha \max_{a \in V}F(a \mid S^{i}) - \epsilon^{i+1}$, then
  \begin{equation*}
    F(S^M) \geq (1-\frac{1}{e^{\alpha}})\max_{|S| \le M}F(S) - \sum_{i=1}^M\epsilon^i.
\end{equation*}
\end{lemma}

To prove Lemma 2, we employ a helpful intermediate observation. We will denote the optimal solution by $S^* \in \arg\max_{|S| \le M}F(S)$.

\begin{lemma}\label{lem:helper}
  If $F(b^{i+1}\mid S^i) \geq \alpha \max_{a \in V}F(a \mid S^{i}) - \epsilon^{i+1}$, then
  \begin{equation*}
    F(b^{i+1} \mid S^i) \geq \frac{\alpha}{M}(F(S^*) - F(S^i)) - \epsilon^{i+1}.
\end{equation*}
\end{lemma}
\begin{proof} \emph{(Lemma~\ref{lem:helper})}.
  Define $T^i = S^*\setminus S^i$ to be the set of all elements that are in $S^*$ but have not yet been selected.
  Order the elements in $T^i$ in an arbitrary order as $t^1, \ldots t^h$ (note that $h \leq M$ because $|S^*| \leq M$. By monotonicity of $F$ and the fact that $S^* \subseteq S^i \union T^i$, it holds that
  \begin{align}
    \label{eq:monot}
    F(S^*) - F(S^i) &\leq F(T^i \union S^i) - F(S^i)\\
    &= \sum_{j=1}^h F(t^j \mid S^i \union \{t^1, \ldots, t^{j-1})\\
    \label{eq:dimret}
    &\leq \sum_{j=1}^hF(t^j \mid S^i)\\
    &\leq M \max_{a \in V\setminus S^i} F(a\mid S^i)\\
    &\leq \frac{M}{\alpha}( F(b^{i+1} \mid S^i) + \epsilon^{i+1})
  \end{align}
  In Equation~(\ref{eq:dimret}), we use diminishing marginal returns. In the end, we use the assumption of the lemma, which implies that
  \begin{equation}
    \label{eq:3}
    \max_{a}F(a \mid S) \leq \frac{1}{\alpha}( F(b^{i+1} \mid S^i) + \epsilon^{i+1}).
  \end{equation}
Rearranging yields the result of the lemma.
\end{proof}

Now we are equipped to prove Lemma~\ref{lem:approxfactor}.
\begin{proof}
  Lemma~\ref{lem:helper} implies that
   \begin{align}
     F(b^{i+1} \mid S^i) &= F(S^{i+1}) - F(S^i)\\
     &\geq \frac{\alpha}{M}\big(F(S^*) - F(S^i)\big) - \epsilon^{i+1}.
   \end{align}
   We rearrange this to
  \begin{align}
    &F(S^*)-F(S^{i+1})\\
    &\;\leq (1-\frac{\alpha}{M})F(S^*) - (1-\frac{\alpha}{M})F(S^i) + \epsilon^{i+1}\\
    &\;= (1-\frac{\alpha}{M})(F(S^*) - F(S^i)) + \epsilon^{i+1}\\
    &\;\leq (1-\frac{\alpha}{M})^{i+1}(F(S^*)-F(S^0)) + \sum_{j=1}^i\big(1 - \frac{\alpha}{M}\big)^{i-j}\epsilon^j \\
    &\;\leq (1-\frac{\alpha}{M})^{i+1}(F(S^*)-F(S^0)) + \sum_{j=1}^i \epsilon^j 
  \end{align}
  With $F(S^0) = F(\emptyset)=0$ we rearrange to
  \begin{align}
    F(S^M) &\geq (1-(1-\frac{\alpha}{M})^M)F(S^*) - \sum_{j=1}^M\epsilon^j\\
    &\geq (1-e^{-\alpha})F(S^*) - \sum_{j=1}^M\epsilon^j.
  \end{align}
\end{proof}
\section{Relative Error}

If we have a monotone but not nonnegative function, we may shift the function and obtain a relative bound:
\begin{lemma} 
\label{lem:relbounds}
Let $F$ be an arbitrary monotone submodular function, and let $\fmin = \min_{S \subseteq V} F(S)$. 
We can define a new, shifted monotone non-negative version of $F$ as $F^+(S) \triangleq F(S) - \fmin$.
If we apply the greedy algorithm to $F^+$, we obtain a solution $\Shat$ that satisfies
$ F^+(\Shat) \geq \alpha F^+(S^*)$, then the solution $\Shat$ has a bounded relative approximation error:
  \begin{align*}
    \frac{F(\Shat) - \fmin}{F(S^*) - \fmin} \geq \alpha.
  \end{align*}
\end{lemma}

\begin{proof}
  By a proof analogous to that of Lemma 2, we get, for $\alpha = (1-\frac{1}{e})$,
 \begin{align}
  F^+(\Shat) &\geq \alpha F^+(S^*)\\
  \Leftrightarrow \quad  F(\Shat) - \fmin &\geq \alpha (F(S^*) - \fmin) \\
  \Leftrightarrow \quad  \frac{F(\Shat) - \fmin}{F(S^*) - \fmin} &\geq \alpha = 1 - \frac{1}{e}.
\end{align}


\end{proof}
\section{Generalization: Upper Envelope Potentials}

In addition to the three specific examples given in the main document (Section 4.1, 4.1, 4.2), 
we can also generalize these constructions to a broad class of HOPs called upper envelope 
potentials~\cite{kohli_cvpr10}. 

Let $G_1,G_2, \ldots, G_q,\ldots,G_b$ be disjoint groups where $b$ is polynomial in the size 
of the number of base variables ($n$). We consider the group count diversity. 
At iteration $t$ in the greedy algorithm, assume without loss of generalization that we have covered 
$G_1,G_2,\ldots,G_k$. Then for the ${(t+1)}^{th}$ iteration, the marginal gain of $\yb$ is:
\begin{equation}
 d(\yb \mid S^t) =  \begin{cases} 0 & \mbox{if } \yb \in \{G_1,..,G_k\} \\ 
 								  1 & \mbox{otherwise.} 
				\end{cases} 
\end{equation}
We can now express $d(\yb \mid S^t)$ as an upper-envelope potential, \ie 
$d(\yb \mid S^t) \equiv \max_q \nu^q(\yb)$ where:
\begin{equation}
\nu^q(\yb) = \mu^q + \sum\limits_{i \in [n]}\sum\limits_{\ell \in L} w_{i\ell}^q \delta_i(\ell)
\label{eqn:uppenv}
\end{equation}
where $\delta_i(\ell)$ returns 1 if $y_i = \ell$. 

Each linear function $\nu^q(\yb)$ encodes two pieces of information:
\begin{enumerate}

\item $\mu^q$ encodes if $G_q$ is uncovered at the end of the $t^{th}$ step, \ie 
\begin{align}
\mu^q = \begin{cases} 0 & \text{if }  q \in \{1,..,k\} \\
					  1 & \text{otherwise.}
		\end{cases}
\end{align} 

\item The second part of \eqnref{eqn:uppenv} indicates whether or not $\yb$ lies in the group $G_q$ 
(defined below). 

\end{enumerate} 

For general groups, it may not be possible to have linear encodings for membership. By construction, 
we describe one practical example:

\paragraph{Region Consistency Diversity.} 
Consider a region $C$ in the image whose superpixels $\yb_C$ we want to bias towards a homogenous or uniform 
labeling. 
Thus, when we search for diverse labelings, we want to encourage the entire set of superpixels $\yb_c$ 
to change together. 
In this case, each group $G_\ell$ corresponds to the set of segmentations that assign label $\ell$ to the region $\yb_c$. 
For such a diversity function, we define $w^q_{i\ell}$ as:
\begin{equation}
 w^q_{i\ell} = \begin{cases} 0 		& \mbox{if } i \in [C] \mbox{ and } \ell=q \\ 
 						-\infty 	& \mbox{otherwise.} 
			\end{cases}
\end{equation}
With this definition of $w^q_{i\ell}$, the second part of \eqnref{eqn:uppenv} indicates 
whether or not $\yb$ lies in the group $G_q$ 

It is known~\cite{rother_cvpr09,kohli_cvpr10} that maximizing any upper-envelope HOP can 
be reduced to the maximization of a pairwise function with the addition of an auxiliary switching 
variable $x$ that takes values from the index set of $q$ (in this case $[L]$). 
\begin{equation}
 \max\limits_{\yb_c}d(\yb \mid S^t) = \max\limits_{\yb_c,x} ( \phi_x(q) + \sum\limits_{i\in V} \phi_{xi}(q,y_i))
\end{equation}
where $\phi_x(q) = \mu^q$ and $\phi_{xi}(q,y_i) = w_{iy_i}^q$.
This pairwise function can be maximised using standard
message passing algorithms such as TRW and BP. However, for some cases, such as the region consistency 
diversity defined above, the pairwise function is supermodular, and graph cuts can be used.
\end{appendix}

\end{document}